\documentclass{article}

\usepackage{geometry}
\usepackage{authblk}
\usepackage{amsmath, amssymb, amsthm, amsfonts, mathrsfs}
\usepackage{graphicx}
\usepackage{booktabs}
\usepackage[numbers]{natbib}

\theoremstyle{definition}
\newtheorem{definition}{Definition}[section]
\theoremstyle{plain}
\newtheorem{theorem}[definition]{Theorem}
\newtheorem{lemma}[definition]{Lemma}
\newtheorem{corollary}[definition]{Corollary}
\newtheorem{assumption}[definition]{Assumption}
\theoremstyle{remark}
\newtheorem*{remark}{Remark}

\numberwithin{equation}{section}

\usepackage{amsmath,amsfonts,bm}

\def\eqref#1{equation~\ref{#1}}

\def\1{\bm{1}}

\def\rx{{\textnormal{x}}}

\def\rvx{{\mathbf{x}}}

\def\vtheta{{\bm{\theta}}}

\def\vb{{\bm{b}}}

\def\vg{{\bm{g}}}
\def\vh{{\bm{h}}}

\def\vu{{\bm{u}}}

\def\vw{{\bm{w}}}
\def\vx{{\bm{x}}}

\def\mA{{\bm{A}}}

\def\mI{{\bm{I}}}

\def\mU{{\bm{U}}}
\def\mV{{\bm{V}}}
\def\mW{{\bm{W}}}

\DeclareMathAlphabet{\mathsfit}{\encodingdefault}{\sfdefault}{m}{sl}
\SetMathAlphabet{\mathsfit}{bold}{\encodingdefault}{\sfdefault}{bx}{n}

\newcommand{\E}{\mathbb{E}}
\newcommand{\Ls}{\mathcal{L}}
\newcommand{\R}{\mathbb{R}}

\DeclareMathOperator*{\argmin}{arg\,min}

\newcommand{\vomega}{\boldsymbol{\omega}}
\newcommand{\uN}{\mathrm{N}}
\newcommand{\uP}{\mathrm{P}}
\newcommand{\ud}{d}
\newcommand{\cB}{\mathcal{B}}
\newcommand{\cF}{\mathcal{F}}
\newcommand{\cG}{\mathcal{G}}
\newcommand{\cH}{\mathcal{H}}
\newcommand{\cJ}{\mathcal{J}}
\newcommand{\cO}{\mathcal{O}}
\newcommand{\cP}{\mathcal{P}}
\newcommand{\bS}{\mathbb{S}}
\newcommand{\ls}{\ell}
\newcommand{\one}[1]{\mathbf{1}_{#1}}

\usepackage[bookmarks]{hyperref}

\begin{document}

\title{A Priori Estimates of the Population Risk for Residual Networks}
\author[1,2,3]{Weinan E}
\author[2]{Chao Ma}
\author[2]{Qingcan Wang}
\affil[1]{Department of Mathematics, Princeton University}
\affil[2]{Program in Applied and Computational Mathematics, Princeton
University}
\affil[3]{Beijing Institute of Big Data Research}
\affil[ ]{\texttt{weinan@math.princeton.edu, \{chaom,qingcanw\}@princeton.edu}}
\date{}
\maketitle

\begin{abstract}
  Optimal a priori estimates are derived for the population risk, also known as
  the generalization error, of a regularized residual network model. An
  important part of the regularized model is the usage of a new path norm,
  called the \emph{weighted path norm}, as the regularization term. The weighted
  path norm treats the skip connections and the nonlinearities differently so
  that paths with more nonlinearities are regularized by larger weights. The
  error estimates are a priori in the sense that the estimates depend only on
  the target function, not on the parameters obtained in the training process.
  The estimates are optimal, in a high dimensional setting, in the sense that
  both the bound for the approximation and estimation errors are comparable to
  the Monte Carlo error rates. A crucial step in the proof is to establish an
  optimal bound for the Rademacher complexity of the residual networks.
  Comparisons are made with existing norm-based generalization error bounds.

\end{abstract}

\section{Introduction}%
\label{sec:intro}

One of the major theoretical challenges in machine learning is to understand, in
a high dimensional setting, the generalization error for deep neural networks,
especially residual networks~\cite{he2016deep} which have become one of the
default choices for many machine learning tasks. Since the networks used in
practice are usually over-parameterized, many recent attempts have been made to
derive bounds that do not deteriorate as the number of parameters grows. In this
regard, the norm-based bounds use some appropriate norms of the parameters to
control the generalization error~\cite{neyshabur2015norm,
bartlett2017spectrally, golowich2017size, barron2018approximation}. Other bounds
based on the idea of compressing the networks~\cite{arora2018stronger} or the
use of the Fisher-Rao information~\cite{liang2017fisher} have also been
proposed. While these generalization bounds differ in many ways, they have one
thing in common: they depend on information about the final parameters obtained
in the training process. Following~\cite{e2018priori}, we call them \emph{a
posteriori} bounds. In this paper, we derive \emph{a priori} estimates of the
population risk for deep residual networks. Compared to the a posteriori
estimates mentioned above, our bounds depend only on the target function and the
network structure (e.g.\ the depth and the width). In addition, our bounds scale
optimally with the network depths and the size of the training data: the
approximation error term scales as $\cO(1/L)$ with the depth $L$, while the
estimation error term scales as $\cO(1/\sqrt{n})$ with the size of the training
data $n$ (independent of the depth), both are comparable to the Monte Carlo
rate.

Our interest in deriving a priori estimates comes from an analogy with finite
element methods (FEM)~\cite{ciarlet2002finite, ainsworth2011posteriori}. Both a
priori and a posteriori error estimates are common in the theoretical analysis
of FEM.\@ In fact, in FEM a priori estimates appeared much earlier and are still
more common than a posteriori estimates~\cite{ciarlet2002finite}, contrary to
the situation in machine learning. Although a priori bounds can not be readily
evaluated due to the fact that the required information about the target
function is not available to us, they provide much insight about the qualitative
behavior of different methods. In the context of machine learning, they also
provide a qualitative comparison between different norms, as we show later.
Most importantly, one can only expect the generalization error to be small for
certain class of target functions, a priori estimates is the most natural way to
encode such information in the error analysis.

The second important point of our approach is to regularize the model. Even
though regularization is quite common in machine learning, neural network models
seem to perform quite well without explicit regularization, as long as one is
good at tuning the hyper-parameters in the training process. For this reason,
there has been some special interest in studying the so-called ``implicit
regularization'' effect. Nevertheless, we feel that the study of properly
regularized models is still of interest, particularly in the over-parametrized
regime, for several reasons:
\begin{enumerate}
  \item These regularized models are much more robust. In other words, one does
    not have to search for the better ones among all the global minimizers using
    excessive tuning. 
  \item They allow us to get an idea about how small the test accuracy can be
    among all the global minimizers.
  \item They can potentially help us to find good minimizers (in terms of test
    accuracy) for the un-regularized model.
\end{enumerate}

For the case of two-layer neural network models, the analytical and practical
advantages of a priori analysis have already been demonstrated
in~\cite{e2018priori}. It was shown there that optimal error rates can be
established for appropriately regularized two-layer neural networks models, and
the accuracy of these models behaves in a much more robust fashion than the
vanilla models without regularization. In this paper, we set out to extend the
work in~\cite{e2018priori} for shallow neural network models to deep ones. We
choose residual network as a starting point.

To derive our a priori estimate, we design a new parameter-based norm for deep
residual networks called the \emph{weighted path norm}, and use this norm as a
regularization term to formulate a regularized problem. Unlike traditional path
norms, our weighted path norm puts more weight on paths that go through more
nonlinearities. In this way, we penalize paths with many nonlinearities and
hence control the complexity of the functions represented by networks with a
bounded norm. By using the weighted path norm as the regularization term, we can
strike a balance between the empirical risk and the complexity of the model, and
thus a balance between the approximation error and the estimation error. This
allows us to prove that the minimizer of the regularized model has the optimal
error rate in terms of the population risk. A comparison with existing
parameter-based norms shows that it is nontrivial to find such balance. 

The rest of the paper is organized as follows. In Section~\ref{sec:sketch}, we
set up the problem and state our main theorem. We also sketch the main ideas in
the proof. In Section~\ref{sec:proof} we give the full proof of the theorems. In
Section~\ref{sec:compare}, we compare our result with related works and put
things into perspective. Conclusions are drawn in Section~\ref{sec:conclusion}.

\paragraph{Notations}
In this paper, we let $\Omega = {[0, 1]}^d$ be the unit hypercube, and consider
target functions with domain $\Omega$. Let $\pi$ be a probability measure on
$\Omega$, for any function $f: \Omega \to \R$, let $\|f\|$ be the $l_2$ norm of
$f$ based on $\pi$,
\begin{equation}
  \|f\|^2 = \int_\Omega f^2(\vx)\pi(d\vx).
\end{equation}
Let $\sigma$ be the ReLU activation function used in the neural network models:
$\sigma(x) = \max\{x, 0\}$. For a vector $\vx$, $\sigma(\vx)$ is a vector of the
same size obtained by applying ReLU component-wise.

\section{Setup of the problem and the main theorem}%
\label{sec:sketch}

\subsection{Setup}

We consider the regression problem and residual networks with ReLU activation
$\sigma(\cdot)$. Assume that the target function $f^*: \Omega \to [0, 1]$. Let
the training set be ${\{(\vx_i, y_i)\}}_{i=1}^n$, where the $\vx_i$'s are
independently sampled from an underlying distribution $\pi$ and $y_i =
f^*(\vx_i)$. Later we will consider problems with noise. 

Consider the following residual network architecture with skip connection in
each layer%
\footnote{In practice, residual networks may use skip connections every several
layers. We consider skip connections every layer for the sake of simplicity. It
is easy to extend the analysis to the more general cases.}
\begin{align}
  \vh_0 & = \mV \vx, \nonumber \\
  \vg_l & = \sigma(\mW_l \vh_{l-1}), \nonumber \\
  \vh_l & = \vh_{l-1} + \mU_l \vg_l,\ l = 1, \dots, L,
  \nonumber \\
  f(\vx; \vtheta) & = \vu^\intercal \vh_L.
  \label{eqn:resnet}
\end{align}
Here the set of parameters $\vtheta = \{\mV, \mW_l, \mU_l, \vu\}$, $\mV \in
\R^{D \times d}$, $\mW_l \in \R^{m \times D}$, $\mU_l \in \R^{D \times m}$, $\vu
\in \R^D$, $L$ is the number of layers, $m$ is the width of the residual blocks
and $D$ is the width of skip connections. Note that we omit the bias term in the
network by assuming that the first element of the input $\vx$ is always 1.

To simplify the proof we will consider the truncated square loss
\begin{equation}
  \ls(\vx; \vtheta) =
  {\big| \mathcal{T}_{[0,1]}f(\vx;\vtheta) - f^*(\vx) \big|}^2,
\end{equation}
where $\mathcal{T}_{[0,1]}$ is the truncation operator: for any function
$h(\cdot)$
\begin{equation}
  \mathcal{T}_{[0,1]}h(\vx) = \min\{\max\{h(\vx), 0\}, 1\}.
\end{equation}
The truncated population risk and empirical risk functions are
\begin{equation}
  \Ls(\vtheta) = \E_{\rvx \sim \pi} \ls(\rvx; \vtheta), \quad
  \hat{\Ls}(\vtheta) = \frac{1}{n} \sum_{i=1}^n \ls(\vx_i; \vtheta),
  \label{eqn:loss}
\end{equation}

\begin{remark}
  The truncation is used in order to simplify the proof for the complexity
  control (Theorem~\ref{thm:rad_resnet}). Other truncation methods can also be
  used. For example, we can truncate the loss function $\ls$, instead of $f$.
\end{remark}

\subsection{Function space and norms}

In this paper, we consider target functions belonging to the Barron space $\cB$.
The following definitions of the Barron space and the corresponding norm are
adopted from~\cite{e2018priori}.

\begin{definition}[Barron space]\label{def:barron}

  Let $\bS^{d-1}$ be the unit sphere in $\mathbb{R}^d$, and $\cF$ be the Borel
  $\sigma$-algebra on $\bS^{d-1}$. For any function $f: \Omega\to \mathbb{R}$,
  define the \emph{Barron norm} of $f$ as
  \begin{equation}
    \|f\|_\cB =
    \inf{\left[ \int_{\bS^{d-1}} |a(\omega)|^2 \pi(d \omega) \right]}^{1/2},
  \end{equation}
  where the infimum is taken over all measurable function $a(\omega)$ and
  probability distribution $\pi$ on $(\bS^{d-1}, \cF)$ that satisfies 
  \begin{equation}
    f(\vx) =
    \int_{\bS^{d-1}} a(\omega) \sigma(\omega^\intercal \vx) \pi(d \omega),
  \end{equation}
  for any $\vx\in\Omega$. 
  
  The \emph{Barron space} $\cB$ is the set of continuous functions with finite
  Barron norm,
  \begin{equation}
    \cB = \{f: \Omega \to \R\ |\ \|f\|_\cB < \infty\}.
  \end{equation}
\end{definition}

The Barron space is large enough to contain many functions of interest. For
example, it was shown in~\cite{klusowski2016risk} that if a function has finite
spectral norm, then it belongs to the Barron space.

\begin{definition}[Spectral norm]\label{def:spectral_norm}
  Let $f \in L^2(\Omega)$, and let $F \in L^2(\R^d)$ be an extension of $f$ to
  $\R^d$, and $\hat{F}$ be the Fourier transform of $F$. Define
  the \emph{spectral norm} of $f$ as
  \begin{equation}
    \gamma(f) = \inf \int_{\R^d}
    \|\vomega\|_1^2 |\hat F(\vomega) | \ud \vomega,
  \end{equation}
  where the infimum is taken over all possible extensions $F$.
\end{definition}

\begin{corollary}
  Let $f: \Omega \to \R$ be a function that satisfies $\gamma(f) < \infty$, then 
  \begin{equation}
    \|f\|_\cB \le \gamma(f) < \infty.
  \end{equation}
\end{corollary}

On the other hand, for residual networks, we define the following
parameter-based norm to control the estimation error. We call this norm the
\emph{weighted path norm} since it is a weighted version of the $l_1$ path norm
studied in~\cite{neyshabur2015path,zheng2018capacity}. 

\begin{definition}[Weighted path norm]\label{def:path_norm}
  Given a residual network $f(\cdot; \vtheta)$ with architecture
  (\ref{eqn:resnet}), define the \emph{weighted path norm} of $f$ as
  \begin{equation}
    \|f\|_\uP = \|\vtheta\|_\uP = \big\| {|\vu|}^\intercal
    (\mI + 3 |\mU_L| |\mW_L|) \cdots (\mI + 3 |\mU_1| |\mW_1|) |\mV| \big\|_1,
    \label{eqn:path_norm}
  \end{equation}
  where $|\mA|$ with $\mA$ being a vector or matrix means taking the
  absolute values of all the entries of the vector or matrix.
\end{definition}

Our weighted path norm is a weighted sum over all paths in the neural network
flowing from the input to the output, and gives larger weight to the paths that
go through more nonlinearities. More precisely, given a path $\cP$, let
$w_1^\cP, w_2^\cP, \dots, w_L^\cP$ be the weights on this path, let $p$ be the
number of non-linearities that $\cP$ goes through. Then, it is straightforward
to see that our weighted path norm can also be expressed as
\begin{equation}
  \|f\|_\uP = \sum_{\cP\ \text{is activated}} 3^p \prod_{l=1}^L |w_l^\cP|.
\end{equation}

\begin{remark}
  The advantage of our weighted path norm can be seen from an ``effective
  depth'' viewpoint. It has been observed that although residual networks can be
  very deep, most information is processed by only a small number of
  nonlinearities. This has been explored for example
  in~\cite{veit2016residual}, where the authors observed numerically that
  residual networks behave like ensembles of networks with fewer layers. Our
  weighted path norm naturally takes this into account. 
\end{remark}

\subsection{Main theorem}

\begin{theorem}[A priori estimate]\label{thm:apriori}
  Let $f^*: \Omega \to [0, 1]$ and assume that the residual network
  $f(\cdot; \vtheta)$ has architecture (\ref{eqn:resnet}). 
  Let $n$ be the number
  of training samples, $L$ be the number of layers and $m$ be the width of the
  residual blocks. Let $\Ls(\vtheta)$ and $\hat\Ls(\vtheta)$ be the truncated
  population risk and empirical risk defined in (\ref{eqn:loss}) respectively;
  let $\|f\|_\cB$ be the Barron norm of $f^*$ and $\|\vtheta\|_\uP$
  be the weighted path norm of $f(\cdot; \vtheta)$ in
  Definition~\ref{def:barron} and~\ref{def:path_norm}. For any $\lambda \ge 4
  + 2 / [3 \sqrt{2 \log(2d)}]$, assume that $\hat\vtheta$ is an optimal solution
  of the regularized model
  \begin{equation}
    \min_\vtheta\ \cJ(\vtheta) :=
    \hat\Ls(\vtheta) + 3\lambda \|\vtheta\|_\uP
    \sqrt{\frac{2 \log(2d)}{n}}.
    \label{eqn:regularity}
  \end{equation}
  Then for any $\delta \in (0, 1)$, with probability at least $1 - \delta$ over
  the random training samples, the population risk satisfies
  \begin{equation}
    \Ls(\hat\vtheta) \le \frac{16 \|f\|_\cB^2}{L m}
    + (12 \|f\|_\cB + 1)
    \frac{3 (4 + \lambda) \sqrt{2 \log(2d)} + 2}{\sqrt n}
    + 4 \sqrt{\frac{2 \log(14 / \delta)}{n}}.
    \label{eqn:apriori}
  \end{equation}
\end{theorem}

\begin{remark}
  \begin{enumerate}
    \item The estimate is a priori in nature since the right hand side of
      (\ref{eqn:apriori}) depends only on the Barron norm of the target function
      without knowing the norm of $\hat\vtheta$.
    \item We want to emphasize that our estimate is nearly optimal. The first
      term in (\ref{eqn:apriori}) shows that the convergence rate with respect
      to the size of the neural network is $\cO(1 / (Lm))$, which matches the
      rate in universal approximation theory for shallow
      networks~\cite{barron1993universal}. The last two terms show that the rate
      with respect to the number of training samples is $\cO(1 / \sqrt n)$,
      which matches the classical estimates of the generalization gap.
    \item The second term depends only on $\|f\|_\cB$ instead of the network
      architecture, thus there is no need to increase the sample size $n$ with
      respect to the network size parameters $L$ and $m$ to ensure that the
      model generalizes well. This is not the case for existing error bounds
      (see Section~\ref{sec:compare}).
  \end{enumerate}
\end{remark}

\subsection{Extension to the case with noise}

Our a priori estimates can be extended to problems with sub-gaussian noise.
Assume that $y_i$ in the training data are given by $y_i = f^*(\vx_i) +
\varepsilon_i$ where $\{\varepsilon_i\}$ are i.i.d.\ random variables such that 
$\E\varepsilon_i = 0$ and
\begin{equation}
  \Pr\{|\varepsilon_i| > t\} \le c e^{-\frac{t^2}{2 \sigma^2}},\quad
  \forall t \ge \tau,
  \label{eqn:noise}
\end{equation}
for some constants $c$, $\sigma$ and $\tau$. Let $\ls_B(\vx; \vtheta) = \ls(\vx;
\vtheta) \wedge B^2$ be the square loss truncated by $B^2$, and define
\begin{equation}
  \Ls_B(\vtheta) = \E_{\rvx \sim \pi} \ls_B(\rvx; \vtheta), \quad
  \hat{\Ls}_B(\vtheta) = \frac{1}{n} \sum_{i=1}^n \ls_B(\vx_i; \vtheta).
  \label{eqn:lossn}
\end{equation}
Then, we have

\begin{theorem}[A priori estimate for noisy problems]\label{thm:apriorin}
  In addition to the conditions in Theorem~\ref{thm:apriori}, assume
  that the noise satisfies (\ref{eqn:noise}). Let $\Ls_B(\vtheta)$ and
  $\hat\Ls_B(\vtheta)$ be the truncated population risk and empirical risk
  defined in (\ref{eqn:lossn}). For $B \ge 1 + \max\{\tau, \sigma \sqrt{\log
  n}\}$ and $\lambda \ge 4 + 2B / [3 \sqrt{2 \log(2d)}]$, assume that
  $\hat\vtheta$ is an optimal solution of the regularized model
  \begin{equation}
    \min_\vtheta\ \cJ(\vtheta) :=
    \hat\Ls(\vtheta) + \lambda B \|\vtheta\|_\uP
    \cdot 3 \sqrt{\frac{2 \log(2d)}{n}}.
    \label{eqn:regularityn}
  \end{equation}
  Then for any $\delta \in (0, 1)$, with probability at least $1 - \delta$ over
  the random training sample, the population risk satisfies
  \begin{equation}
    \Ls(\hat\vtheta) \le \frac{16 \|f\|_\cB^2}{L m}
    + (12 \|f\|_\cB + 1)
    \frac{3 (4 + \lambda) B \sqrt{2 \log(2d)} + 2 B^2}{\sqrt n}
    + 4 B^2 \sqrt{\frac{2 \log(14 / \delta)}{n}}
    + \frac{2 c (4 \sigma^2 + 1)}{\sqrt n}.
    \label{eqn:apriorin}
  \end{equation}
\end{theorem}

We see that the a priori estimates for problems with noise only differ from that
for problems without noise by a logarithmic term. In particular, the estimates
of the generalization error are still nearly optimal.

\subsection{Proof sketch}

We prove the main theorem in 3 steps. We list the main intermediate results in
this section, and leave the full proof to Section~\ref{sec:proof}.

First, we show that any function $f$ in the Barron space can be approximated by
residual networks with increasing depth or width, and with weighted path norm
uniformly bounded. 

\begin{theorem}\label{thm:approx}
  For any target function $f^*\in\cB$, and any $L, m\geq1$, there
  exists a residual network $f(\cdot; \tilde\vtheta)$ with depth $L$ and width
  $m$, such that
  \begin{equation}
    \|f(\vx; \tilde\vtheta)-f^*\|^2
    \le \frac{16\|f^*\|_\cB^2}{L m}
    \label{eqn:approx}
  \end{equation}
  and 
  \[
    \|\tilde\vtheta\|_\uP \le 12 \|f^*\|_\cB.
  \]
\end{theorem}

Secondly, we show that the weighted path norm helps to bound the Rademacher
complexity. Since the Rademacher complexity can bound the generalization gap,
this gives an a posteriori bound on the generalization error.

Recall the definition of Rademacher complexity:

\begin{definition}[Rademacher complexity]
  Given a function class $\cH$ and sample set $S = {\{x_i\}}_{i=1}^n$, the
  \emph{(empirical) Rademacher complexity} of $\cH$ with respect to $S$ is
  defined as
  \begin{equation}
    \hat R(\cH) = \frac{1}{n}
    \E_\xi \left[ \sup_{h \in \cH} \sum_{i=1}^n \xi_i h(x_i) \right],
    \label{eqn:rademacher}
  \end{equation}
  where the $\xi_i$'s are independent random variables with $\Pr\{\xi_i = 1\} =
  \Pr\{\xi_i = -1\} = 1/2$.
\end{definition}

It is well-known that the Rademacher complexity can be used to control the
generalization gap~\cite{shalev2014understanding}.

\begin{theorem}\label{thm:rademacher}
  Given a function class $\cH$, for any $\delta \in (0, 1)$, with probability at
  least $1 - \delta$ over the random samples ${\{x_i\}}_{i=1}^n$,
  \begin{equation}
    \sup_{h \in \cH}
    \left|\E_\rx[h(\rx)] - \frac{1}{n}\sum_{i=1}^n h(x_i)\right|
    \le 2 \hat R(\cH) + 2 \sup_{h, h' \in \cH} {\|h - h'\|}_\infty
    \sqrt{\frac{2 \log(4 / \delta)}{n}}.
  \end{equation}
\end{theorem}

The following theorem is a crucial step in our analysis. It shows that the
Rademacher complexity of residual networks can be controlled by the weighted
path norm.

\begin{theorem}\label{thm:rad_resnet}
  Let $\cF^Q = \{f(\cdot; \vtheta): \|\vtheta\|_\uP \le Q\}$ where the $f(\cdot,
  \vtheta)$'s are residual networks defined by (\ref{eqn:resnet}). Assume that
  the samples ${\{\vx_i\}}_{i=1}^n \subset \Omega$, then we have
  \begin{equation}
    \hat R(\cF^Q) \le 3 Q \sqrt{\frac{2 \log(2d)}{n}}.
    \label{eqn:rad_resnet}
  \end{equation}
\end{theorem}

Note that the definition of $\cF^Q$ does not specify the depth or width of the
network. Consequently our Rademacher complexity bound does not depend on the
depth and width of the network. Hence, the resulted a-posteriori estimate has
no dependence on $L$ and $m$ either. 

\begin{theorem}[A posteriori estimates]\label{thm:apost}
  Let $\|\vtheta\|_\uP$ be the weighted path norm of residual network $f(\cdot;
  \vtheta)$. Let $n$ be the number of training samples. Let $\Ls(\vtheta)$ and
  $\hat\Ls(\vtheta)$ be the truncated population risk and empirical risk defined
  in (\ref{eqn:loss}). Then for any $\delta \in (0, 1)$, with probability at
  least $1 - \delta$ over the random training samples, we have
  \begin{equation}
    \left| \Ls(\vtheta) - \hat\Ls(\vtheta) \right|
    \le 2 (\|\vtheta\|_\uP + 1) \frac{6 \sqrt{2 \log(2d)} + 1}{\sqrt n}
    + 2 \sqrt{\frac{2 \log(7 / \delta)}{n}}.
    \label{eqn:apost}
  \end{equation}
\end{theorem}

Consider the decomposition
\begin{equation}
  \Ls(\hat\vtheta) - \Ls(\tilde\vtheta)
  = \left[ \Ls(\hat\vtheta) - \cJ(\hat\vtheta) \right]
  + \left[ \cJ(\hat\vtheta) - \cJ(\tilde\vtheta) \right]
  + \left[ \cJ(\tilde\vtheta) -\Ls(\tilde\vtheta) \right].
  \label{eqn:l_decomp0}
\end{equation}
Recall that $\hat\vtheta$ is the optimal solution of the minimization problem
(\ref{eqn:regularity}), and $\tilde\vtheta$ corresponds to the approximation in
Theorem~\ref{thm:approx}.

By the definition of $\cJ$ (\ref{eqn:regularity}),
\begin{align*}
  \Ls(\hat\vtheta) - \cJ(\hat\vtheta)
  & \le \left| \Ls(\hat\vtheta) - \hat\Ls(\hat\vtheta) \right|
  - 3 \lambda \|\hat\vtheta\|_\uP \sqrt{\frac{2 \log(2d)}{n}}, \\
  \cJ(\tilde\vtheta) -\Ls(\tilde\vtheta)
  & \le \left| \Ls(\tilde\vtheta) - \hat\Ls(\tilde\vtheta) \right|
  + 3 \lambda \|\tilde\vtheta\|_\uP \sqrt{\frac{2 \log(2d)}{n}}.
\end{align*}
From the a posteriori estimate (\ref{eqn:apost}), both $|\Ls(\hat\vtheta) -
\hat\Ls(\hat\vtheta)|$ and $|\Ls(\tilde\vtheta) - \hat\Ls(\tilde\vtheta)|$ are
bounded with high probability, thus both $\Ls(\hat\vtheta) - \cJ(\hat\vtheta)$
and $\cJ(\tilde\vtheta) -\Ls(\tilde\vtheta)$ are bounded with high probability.
In addition, $\cJ(\hat\vtheta) - \cJ(\tilde\vtheta) \le 0$, and the
approximation result (\ref{eqn:approx}) bounds $\Ls(\tilde\vtheta)$. Plugging
all of the above into (\ref{eqn:l_decomp0}) will give us the a priori estimates
in Theorem~\ref{thm:apriori}.

For problems with noise, we can similarly bound $\Ls_B(\vtheta) -
\cJ(\vtheta)$ instead of $\Ls(\vtheta) - \cJ(\vtheta)$. Hence, to formulate an a
priori estimate, we also need to control $\Ls(\vtheta) - \Ls_B(\vtheta)$. This
is given by the following lemma:

\begin{lemma}\label{lem:noise}
  Assume that the noise $\varepsilon$ has zero mean and satisfies
  (\ref{eqn:noise}), and $B \ge 1 + \max\left\{ \tau, \sigma \sqrt{\log n}
  \right\}$. For any
  $\vtheta$ we have
  \begin{equation}
    \left| \Ls(\vtheta) - \Ls_B(\vtheta) \right|
    \le \frac{c (4 \sigma^2 + 1)}{\sqrt n}.
  \end{equation}
\end{lemma}

\section{Proof}%
\label{sec:proof}

\subsection{Approximation error}

For the approximation error,~\cite{e2018priori} proved the following result
for shallow networks.

\begin{theorem}\label{thm:approx_shallow}
  For any target function $f^* \in \cB$ and any $M \ge 1$,
  there exists a two-layer network with width $M$, such that
  \begin{equation}
    \left\| \sum_{j=1}^M a_j \sigma(\vb_j^\intercal \rvx) - f^*(\rvx) \right\|^2
    \le \frac{16\|f^*\|_\cB^2}{M}
  \end{equation}
  and
  \begin{equation}
    \sum_{j=1}^M |a_j| {\|\vb_j\|}_1 
    \le 4 \|f^*\|_\cB.
  \end{equation}
\end{theorem}
We have omitted writing out the bias term. This can be accommodated by assuming
that the first element of input $\vx$ is always 1. For residual networks, we
prove the approximation result (Theorem~\ref{thm:approx}) by splitting the
shallow network into several parts and stack them
vertically~\cite{e2018exponential}. This is allowed by the special structure of
residual networks.

\begin{proof}[Proof of Theorem~\ref{thm:approx}]
   We construct a residual network $f(\cdot; \tilde\vtheta)$ with input
  dimension $d$, depth $L$, width $m$, and $D = d + 1$ using
  \begin{gather*}
    \mV = \begin{bmatrix}
      \mI_d & 0
    \end{bmatrix}^\intercal, \quad
    \vu = \begin{bmatrix}
      0 & 0 & \cdots & 0 & 1
    \end{bmatrix}^\intercal, \\
    \mW_l = \begin{bmatrix}
      \vb_{(l-1) m + 1}^\intercal & 0 \\
      \vb_{(l-1) m + 2}^\intercal & 0 \\
      \vdots & \vdots \\
      \vb_{l m}^\intercal & 0 \\
    \end{bmatrix}, \quad
    \mU_l = \begin{bmatrix}
    0 & 0 & \cdots & 0 \\
    \vdots & \vdots & \ddots & \vdots \\
    0 & 0 & \cdots & 0 \\
    a_{(l-1) m + 1} & a_{(l-1) m + 2} & \cdots & a_{l m} \\
    \end{bmatrix} \\
  \end{gather*}
  for $l = 1, \dots, L$. Then it is easy to verify that $f(\vx; \tilde\vtheta) =
  \sum_{j=1}^{L m} a_j \sigma(\vb_j^\intercal \vx)$, and
  \[
    \|\tilde\vtheta\|_\uP
    = 3 \sum_{j=1}^{L m} |a_j| \|\vb_j\|_1 \le 12 \|f^*\|_\cB.
  \]
\end{proof}

\subsection{Rademacher complexity}

We use the method of induction to bound the Rademacher complexity of residual
networks. We first extend the definition of weighted path norm to hidden neurons
in the residual network. 

\begin{definition}
  Given a residual network defined by (\ref{eqn:resnet}), recall the definition
  of $\vg_l$,
  \begin{equation}
    \vg_l(\vx) = \sigma(\mW_l \vh_{l-1}), \quad l = 1, \dots, L.
    \label{eqn:resnet_g}
  \end{equation}
  Let $g_l^i$ be the $i$-th element of $\vg_l$, define the \emph{weighted path
  norm}
  \begin{equation}
    \|g_l^i\|_\uP = \left\| 3 |\mW_l^{i,:}|
    (\mI + 3 |\mU_{l-1}| |\mW_{l-1}|) \cdots (\mI + 3 |\mU_1| |\mW_1|)
    |\mV| \right\|_1,
    \label{eqn:path_norm_g}
  \end{equation}
  where $\mW_l^{i,:}$ is the $i$-th row of $\mW_l$.
\end{definition}

The following lemma establishes the relationship between
$\|f\|_\uP$ and $\|g_l^i\|_\uP$. Lemma~\ref{lem:g_class} gives properties of the
corresponding function class. 

\begin{lemma}\label{lem:pathnorm}
  For the weighted path norm defined in (\ref{eqn:path_norm}) and
  (\ref{eqn:path_norm_g}), we have
  \begin{equation}\label{eq:f_decomp}
    \|f\|_\uP = \sum_{l=1}^L \sum_{j=1}^m
    \left( {|\vu|}^\intercal |\mU_l^{:,j}| \right) \|g_l^j\|_\uP
    + \big\| {|\vu|}^\intercal |\mV| \big\|_1,
  \end{equation}
  and
  \begin{equation}\label{eq:g_decomp}
    \|g_l^i\|_\uP = \sum_{k=1}^l \sum_{j=1}^m
    3 \left( |\mW_l^{i,:}| |\mU_k^{:,j}| \right) \|g_k^j\|_\uP
    + 3 \big\| |\mW_l^{i,:}| |\mV| \big\|_1,
  \end{equation}
  where $\mU_l^{:,j}$ is the $j$-th column of $\mU_l$.
\end{lemma}

\begin{proof}
Recall the definition of $\|f\|_\uP$, we have
\begin{align*}
  \|f\|_\uP
  & = \big\| {|\vu|}^\intercal (\mI + 3 |\mU_L| |\mW_L|)
  \cdots (\mI + 3 |\mU_1| |\mW_1|) |\mV| \big\|_1 \\
  & = \left\| \sum_{l=1}^L |\vu|^\intercal |\mU_l|
  \cdot 3|\mW_l| \prod_{j=1}^{l-1} (\mI + 3|\mU_{l-j}| |\mW_{l-j}|) |\mV|
  + |\vu|^\intercal |\mV| \right\|_1 \\
  & = \sum_{l=1}^L \sum_{j=1}^m
  \left( {|\vu|}^\intercal |\mU_l^{:,j}| \right) \|g_l^j\|_\uP
  + \big\| {|\vu|}^\intercal |\mV| \big\|_1,
\end{align*}
which gives (\ref{eq:f_decomp}). Similarly we obtain (\ref{eq:g_decomp}).
\end{proof}
\ 

\begin{lemma}\label{lem:g_class}
  Let $\cG_l^Q = \{g_l^i: \|g_l^i\|_\uP \le Q\}$, then
  \begin{enumerate}
    \item $\cG_k^Q \subseteq \cG_l^Q$ for $k \le l$;
    \item $\cG_l^q \subseteq \cG_l^Q$ and $\cG_l^q = \frac{q}{Q} \cG_l^Q$ for $q
      \le Q$.
  \end{enumerate}
\end{lemma}

\begin{proof}
  For any $g_k \in \cG_k^Q$, let $\mV$, ${\{\mU_j, \mW_j\}}_{j=1}^k$ and $\vw$
  be the parameters of $g_k$, where $\vw$ is the vector of the parameters in the
  output layer (the $\mW_k^{i,:}$ in the definition of $g_l^i$). Then, for any
  $l \ge k$, consider $g_l$ generated by parameters $\mV$,
  ${\{\mU_j,\mW_j\}}_{j=1}^l$ and $\vw$, with $\mU_j = 0$ and $\mW_j = 0$ for
  any $k < j \le l$. Now it is easy to verify that $g_l = g_k$ and $\|g_l\|_\uP
  = \|g_k\|_\uP \le Q$. Hence, we have $\cG_k^Q \subseteq \cG_l^Q$.

  On the other hand, obviously we have $\cG_l^q \subseteq \cG_l^Q$ for any
  $q\leq Q$. For any $g_l\in\cG_l^q$, define $\tilde{g}_l$ by replacing the
  output parameters $\vw$ by $\frac{Q}{q}\vw$, then we have
  $\|\tilde{g}_l\|_\uP=\frac{Q}{q}\|g_l\|_\uP\leq Q$, and hence
  $\tilde{g}_l\in\cG_l^Q$. Therefore, we have
  $\frac{Q}{q}\cG_l^q\subseteq\cG^Q$. Similarly we can obtain
  $\frac{q}{Q}\cG_l^Q\subseteq\cG^q$. Consequently, we have $\cG_l^q =
  \frac{q}{Q} \cG_l^Q$.
\end{proof}

We will also use the following two lemmas about Rademacher
complexity~\cite{shalev2014understanding}. Lemma~\ref{lem:rad_linear} bounds the
Rademacher complexity of linear functions, and Lemma~\ref{lem:contraction} gives
the contraction property of the Rademacher complexity. 

\begin{lemma}\label{lem:rad_linear}
  Let $\cH = \{h(\vx) = \vu^\intercal \vx: \|\vu\|_1 \le 1\}$. Assume that the
  samples ${\{\vx_i\}}_{i=1}^n \subset \R^d$, then
  \begin{equation}
    \hat R(\cH) \le \max_i {\|\vx_i\|}_\infty \sqrt{\frac{2 \log(2d)}{n}}.
  \end{equation}
\end{lemma}

\begin{lemma}\label{lem:contraction}
  Assume that $\phi_i, i = 1, \dots, n$ are Lipschitz continuous functions with
  uniform Lipschitz constant $L_\phi$, i.e., $|\phi_i(x) - \phi_i(x')| \le
  L_\phi |x - x'|$ for $i = 1, \dots, n$, then
  \begin{equation}
    \E_\xi \left[ \sup_{h \in \cH} \sum_{i=1}^n \xi_i \phi_i(h(x_i)) \right]
    \le L_\phi \E_\xi \left[ \sup_{h \in \cH} \sum_{i=1}^n \xi_i h(x_i) \right].
  \end{equation}
\end{lemma}

With Lemma~\ref{lem:pathnorm}--\ref{lem:contraction}, we can come to prove
Theorem~\ref{thm:rad_resnet}.

\begin{proof}[Proof of Theorem~\ref{thm:rad_resnet}]
  We first estimate the Rademacher complexity of $\cG_l^Q$,
  \begin{equation}\label{eqn:rad_g}
    \hat{R}(\cG_l^Q)\leq Q\sqrt{\frac{2\log(2d)}{n}}.
  \end{equation}
  This is done by induction. By definition, $g_1^i(\vx)=\sigma(\mW_1^{i,:} \mV
  \vx)$. Hence, using Lemma~\ref{lem:rad_linear} and~\ref{lem:contraction}, we
  conclude that the statement (\ref{eqn:rad_g}) holds for $l=1$. Now, assume
  that the result holds for $1, 2, \dots, l$. Then, for $l + 1$ we have
  \begin{align*}
    n \hat R(\cG_{l+1}^Q)
    & = \E_\xi \sup_{g_{l+1} \in \cG_{l+1}^Q}
    \sum_{i=1}^n \xi_i g_{l+1}(\vx_i) \\
    & = \E_\xi \sup_{(1)} \sum_{i=1}^n \xi_i \sigma(\vw_l^\intercal
    (\mU_l \vg_l + \mU_{l-1} \vg_{l-1} + \cdots + \mU_1 \vg_1 + \vh_0)) \\
    & \le \E_\xi \sup_{(1)} \sum_{i=1}^n \xi_i (\vw_{l+1}^\intercal
    (\mU_l \vg_l + \mU_{l-1} \vg_{l-1} + \cdots + \mU_1 \vg_1 + \vh_0)) \\
    & \le \E_\xi \sup_{(2)} \left\{
    \sum_{k=1}^l a_k \sup_{g \in \cG_k^1}
    \left| \sum_{i=1}^n\xi_i g(\vx_i) \right|
    + b \sup_{\|\vu\|_1 \le 1}
    \left| \sum_{i=1}^n \xi_i \vu^\intercal \vx_i \right| \right\} \\
    & \le \E_\xi \sup_{\substack{a + b \le \frac{Q}{3} \\ a, b \ge 0}}
    \left\{ a \sup_{g \in \cG_l^1} \left| \sum_{i=1}^n \xi_i g(\vx_i) \right|
    + b \sup_{\|\vu\|_1 \le 1}
    \left| \sum_{i=1}^n \xi_i \vu^\intercal \vx_i \right| \right\}\\
    & \le \frac{Q}{3} \left[
    \E_\xi \sup_{g \in \cG_l^1} \left| \sum_{i=1}^n \xi_i g(\vx_i) \right|
    + \E_\xi \sup_{\|\vu\|_1 \le 1}
    \left| \sum_{i=1}^n \xi_i \vu^\intercal \vx_i \right| \right]\\
  \end{align*}
  where condition (1) is $\sum\limits_{k=1}^{l}\sum\limits_{j=1}^m 3 \left(
  |\vw_{l+1}|^\intercal |U_k^{:,j}| \right) \|g_k^j\|_\uP + 3 \left\|
  |\vw_{l+1}|^\intercal |\mV| \right\|_1 \le Q$, and condition (2) is $3
  \sum_{k=1}^l a_k + 3b \le Q$. The first inequality is due to the contraction
  lemma, while the third inequality is due to Lemma~\ref{lem:g_class}. On the
  one hand, we have
  \[
    \E_\xi \sup_{\|\vu\|_1 \le 1}
    \left| \sum_{i=1}^n \xi_i \vu^\intercal \vx_i \right|
    = \E_\xi \sup_{\|\vu\|_1 \le 1} \sum_{i=1}^n \xi_i \vu^\intercal \vx_i
    \le n \sqrt{\frac{2 \log(2d)}{n}}.
  \]
  On the other hand, since $0 \in \cG_l^1$, for any $\{\xi_1, \dots, \xi_n\}$, we have
  \[
  \sup_{g \in \cG_l^1} \sum_{i=1}^n \xi_i g(\vx_i) \geq0.
  \]
  Hence, we have
  \begin{align*}
    \sup_{g \in \cG_l^1} \left| \sum_{i=1}^n \xi_i g(\vx_i) \right|
    & \le \max \left\{ \sup_{g \in \cG_l^1} \sum_{i=1}^n \xi_i g(\vx_i),\ 
    \sup_{g \in \cG_l^1} \sum_{i=1}^n - \xi_i g(\vx_i) \right\} \\
    & \le \sup_{g \in \cG_l^1} \sum_{i=1}^n \xi_i g(\vx_i)
    + \sup_{g \in \cG_l^1} \sum_{i=1}^n -\xi_i g(\vx_i),
  \end{align*}
  which gives
  \[
    \E_\xi \sup_{g \in \cG_l^1} \left| \sum_{i=1}^n \xi_i g(\vx_i) \right|
    \le 2 \E_\xi \sup_{g \in \cG_l^1} \sum_{i=1}^n \xi_i g(\vx_i)
    = 2n \hat R(\cG_l^1).
  \]
  Therefore, we have
  \[
    \hat R(\cG_{l+1}^Q)
    \le \frac{Q}{3} \left[ 2 \sqrt{\frac{2 \log(2d)}{n}}
    + \sqrt{\frac{2 \log(2d)}{n}} \right]
    \le Q \sqrt{\frac{2 \log(2d)}{n}}.
  \]

  Similarly, based on the control for the Rademacher complexity of $\cG_1^Q,
  \dots, \cG_L^Q$, we get
  \[
    \hat R(\cF^Q) \le 3Q \sqrt{\frac{2 \log(2d)}{n}}.
  \]
\end{proof}

\subsection{A posteriori estimates}

\begin{proof}[Proof of Theorem~\ref{thm:apost}]
  Let $\cH = \left\{ \ls(\cdot; \vtheta): \|\vtheta\|_\uP \le Q \right\}$. Notice
  that for all $\vx$,
  \[
    |\ls(\vx; \vtheta) - \ls(\vx; \vtheta')|
    \le 2 |f(\vx; \vtheta) - f(\vx; \vtheta')|.
  \]
  By Lemma~\ref{lem:contraction},
  \[
    \hat R(\cH) = \frac{1}{n} \E_\xi \left[
    \sup_{\|\vtheta\|_\uP \le Q} \sum_{i=1}^n \xi_i \ls(\vx_i; \vtheta) \right]
    \le \frac{2}{n} \E_\xi \left[
    \sup_{\|\vtheta\|_\uP \le Q} \sum_{i=1}^n \xi_i f(\vx_i; \vtheta) \right]
    = 2 \hat R(\cF^Q).
  \]
  From Theorem~\ref{thm:rademacher}, with probability at least $1 - \delta$,
  \begin{align}
    \sup_{\|\vtheta\|_\uP \le Q} \left| \Ls(\vtheta) - \hat\Ls(\vtheta) \right|
    & \le 2 \hat R(\cH) + 2 \sup_{h, h' \in \cH} \|h - h'\|_\infty
    \sqrt{\frac{2 \log(4 / \delta)}{n}} \nonumber \\
    & \le 12 Q \sqrt{\frac{2 \log(2d)}{n}}
    + 2 \sqrt{\frac{2 \log(4 / \delta)}{n}}.
  \end{align}

  Now take $Q = 1, 2, 3, \dots$ and $\delta_Q = \frac{6 \delta}{{(\pi Q)}^2}$,
  then with probability at least $1 - \sum_{Q=1}^\infty \delta_Q = 1 - \delta$,
  the bound
  \[
    \sup_{\|\vtheta\|_\uP \le Q} \left| \Ls(\vtheta) - \hat\Ls(\vtheta) \right|
    \le 12 Q \sqrt{\frac{2 \log(2d)}{n}}
    + 2 \sqrt{\frac{2}{n} \log\frac{2 {(\pi Q)}^2}{3 \delta}}
  \]
  holds for all $Q \in \mathbb{N}^*$. In particular, for given $\vtheta$, the
  inequality holds for $Q = \left\lceil \|\vtheta\| \right\rceil <
  \|\vtheta\|_\uP + 1$, thus
  \begin{align*}
    \left| \Ls(\vtheta) - \hat\Ls(\vtheta) \right|
    & \le 12 (\|\vtheta\|_\uP + 1) \sqrt{\frac{2 \log(2d)}{n}}
    + 2 \sqrt{\frac{2}{n} \log\frac{7 {(\|\vtheta\|_\uP + 1)}^2}{\delta}} \\
    & \le 12 (\|\vtheta\|_\uP + 1) \sqrt{\frac{2 \log(2d)}{n}}
    + 2 \left[ \frac{\|\vtheta\|_\uP + 1}{\sqrt n}
    + \sqrt{\frac{2 \log(7 / \delta)}{n}} \right] \\
    & = 2 (\|\vtheta\|_\uP + 1) \frac{6 \sqrt{2 \log(2d)} + 1}{\sqrt n}
    + 2 \sqrt{\frac{2 \log(7 / \delta)}{n}}.
  \end{align*}
\end{proof}

\subsection{A priori estimates}

Now we are ready to prove the main Theorem~\ref{thm:apriori}.

\begin{proof}[{Proof of Theorem~\ref{thm:apriori}}]
  Let $\hat\vtheta$ be the optimal solution of the regularized model
  (\ref{eqn:regularity}), and $\tilde\vtheta$ be the approximation in
  Theorem~\ref{thm:approx}. Consider
  \begin{equation}
    \Ls(\hat\vtheta) = \Ls(\tilde\vtheta)
    + \left[ \Ls(\hat\vtheta) - \cJ(\hat\vtheta) \right]
    + \left[ \cJ(\hat\vtheta) - \cJ(\tilde\vtheta) \right]
    + \left[ \cJ(\tilde\vtheta) -\Ls(\tilde\vtheta) \right].
    \label{eqn:l_decomp}
  \end{equation}

  From (\ref{eqn:approx}) in Theorem~\ref{thm:approx}, we have
  \begin{equation}
    \Ls(\tilde\vtheta) \le \frac{16 \|f^*\|_\cB^2}{Lm}.
    \label{eqn:approx2}
  \end{equation}
  Compare the definition of $\cJ$ in (\ref{eqn:regularity}) and the gap $\Ls -
  \hat\Ls$ in (\ref{eqn:apost}), with probability at least $1 - \delta / 2$,
  \begin{align}
    \Ls(\hat\vtheta) - \cJ(\hat\vtheta)
    & \le \left( \|\hat\vtheta\|_\uP + 1 \right)
    \frac{3 (4 - \lambda) \sqrt{2 \log(2d)} + 2}{\sqrt n}
    + 3 \lambda \sqrt{\frac{2 \log(2d)}{n}}
    + 2 \sqrt{\frac{2 \log(14 / \delta)}{n}} \nonumber \\
    & \le 3 \lambda \sqrt{\frac{2 \log(2d)}{n}}
    + 2 \sqrt{\frac{2 \log(14 / \delta)}{n}}
    \label{eqn:l-j_hat}
  \end{align}
  since $\lambda \ge 4 + 2 / [3 \sqrt{2 \log(2d)}]$;
  with probability at least $1 - \delta / 2$, we have
  \begin{equation}
    \cJ(\tilde\vtheta) - \Ls(\tilde\vtheta)
    \le \left( \|\tilde\vtheta\|_\uP + 1 \right)
    \frac{3 (4 + \lambda) \sqrt{2 \log(2d)} + 2}{\sqrt n}
    - 3 \lambda \sqrt{\frac{2 \log(2d)}{n}}
    + 2 \sqrt{\frac{2 \log(14 / \delta)}{n}}
    \label{eqn:j-l_tilde}
  \end{equation}
  Thus with probability at least $1 - \delta$, (\ref{eqn:l-j_hat}) and
  (\ref{eqn:j-l_tilde}) hold simultaneously. In addition, we have
  \begin{equation}
    \cJ(\hat\vtheta) - \cJ(\tilde\vtheta) \le 0
    \label{eqn:j_hat-j_tilde}
  \end{equation}
  since $\hat\vtheta = \argmin_\vtheta \cJ(\vtheta)$.

  Now plugging (\ref{eqn:approx2}--\ref{eqn:j_hat-j_tilde}) into
  (\ref{eqn:l_decomp}), and noticing that $\|\tilde\vtheta\|_\uP \le 12
  \|f^*\|_\cB$ from Theorem~\ref{thm:approx}, we see that the main
  theorem (\ref{eqn:apriori}) holds with probability at least $1 - \delta$.
\end{proof}

Finally, we deal with the case with noise and prove Theorem~\ref{thm:apriorin}.
For problems with noise, we decompose $\Ls(\hat\vtheta) - \Ls(\tilde\vtheta)$ as
\begin{align}
  \Ls(\hat\vtheta) - \Ls(\tilde\vtheta)
  =& \left[ \Ls(\hat\vtheta) - \Ls_B(\hat\vtheta)\right]
  + \left[ \Ls_B(\hat\vtheta) - \cJ_B(\hat\vtheta) \right]
  + \left[ \cJ_B(\hat\vtheta) - \cJ_B(\tilde\vtheta) \right]\nonumber \\
  &+ \left[ \cJ_B(\tilde\vtheta) - \Ls_B(\tilde\vtheta)\right]
  + \left[ \Ls_B(\tilde\vtheta) - \Ls(\tilde\vtheta) \right].
\label{eqn:noise_decomp}
\end{align}

Based on the results we had for the case without noise, in
(\ref{eqn:noise_decomp}) we only have to estimate the first and the last terms.
This is given by Lemma~\ref{lem:noise}. Finally, we prove Lemma~\ref{lem:noise}.

\begin{proof}[Proof of Lemma~\ref{lem:noise}]
  Let $Z = f(\vx; \vtheta) - f^*(\vx) - \varepsilon$, then we have
  \begin{align*}
    \left| \Ls(\vtheta) - \Ls_B(\vtheta) \right|
    &= \E \left[ (Z^2 - B^2) \one{|Z| \ge B} \right] \\
    &= \int_0^\infty \Pr \left\{ Z^2 - B^2 \ge t^2 \right\} \ud t^2 \\
    &= \int_0^\infty \Pr \left\{ |Z| \ge \sqrt{B^2 + t^2} \right\} \ud t^2.
  \end{align*}
  As $0 \le f(\vx; \vtheta) \le 1$ and $0 \le f^*(\vx; \vtheta) \le 1$, we
  have
  \[
    \int_0^\infty \Pr \left\{ |Z| \ge \sqrt{B^2 + t^2} \right\} \ud t^2 \le
    \int_0^\infty \Pr \left\{ |\varepsilon| \ge \sqrt{B^2 + t^2} - 1 \right\}
    \ud t^2.
  \]
  Let $s = \sqrt{B^2 + t^2}$, then
  \begin{align*}
    \int_0^\infty
    & \Pr \left\{ |\varepsilon| \ge \sqrt{B^2 + t^2} - 1 \right\} \ud t^2
    \le \int_B^\infty c e^{-\frac{{(s - 1)}^2}{2 \sigma^2}} \ud s^2 \\
    & = \int_{B-1}^\infty 2 c e^{-\frac{s^2}{2 \sigma^2}} \ud s^2
    + \int_{B-1}^\infty 4 c e^{-\frac{s^2}{2 \sigma^2}} \ud s \\
    & \le 4 c \sigma^2 e^{-\frac{{(B - 1)}^2}{2 \sigma^2}}
    + \sqrt{\frac{2}{\pi}} c e^{-\frac{{(B - 1)}^2}{2 \sigma^2}}\\
    & \le \frac{c (4 \sigma^2 + 1)}{\sqrt n}.
  \end{align*}
\end{proof}

\section{Comparison with norm-based a posteriori estimates}%
\label{sec:compare}

Different norms have been used as a vehicle to bound the generalization error of
deep neural networks, including the group norm and path norm given
in~\cite{neyshabur2015norm}, the spectral norm in~\cite{bartlett2017spectrally},
and the variational norm in~\cite{barron2018approximation}. In these works, the
bounds for the generalization gap $\Ls(\vtheta) - \hat\Ls(\vtheta)$ is derived
from a Rademacher complexity bound of the set $\cF^Q = \{f(\vx; \vtheta):
\|\vtheta\|_\uN \le Q\}$, as in Theorem~\ref{thm:rad_resnet}, where
$\|\vtheta\|_\uN$ is some norm or value computed from the parameter $\vtheta$.
These estimates are a posteriori estimates. They are shown to be valid once the
complexity of $\cF^Q$ is controlled.

However, finding a set of functions with small complexity is not enough
to explain the generalization of neural networks. The population risk contains
two parts---the approximation error and the estimation error. In general, the
approximation error bounds require the hypothesis space to be large enough and
the estimation error bounds require the hypothesis space to be small enough. A
posteriori estimates only deal with the estimation error. In a priori estimates,
both effects are present and we have to strike a balance between approximation
and estimation. In this sense, a priori estimates can better reflect the quality
of the norm or the hypothesis space selected. Therefore in order to compare our
estimates with previous results, we turn the previous a posteriori estimates
into a priori estimates by building approximation error bounds for the other
approaches that have been proposed in the same way as we did for ours. These
approximation error bounds allow us to translate existing a posteriori estimates
to a priori estimates and thereby put previous results on the same footing as
ours. 

To start with, based on the analysis in Section~\ref{sec:sketch}
and~\ref{sec:proof}, we provide a general framework for establishing a priori
estimates from norm-based a posteriori estimates. It holds for both residual
networks and deep fully-connected networks:
\begin{equation}
  f(\vx; \vtheta) = \mW_L \sigma(\mW_{L-1} \sigma(\cdots \sigma(\mW_1 \vx)))
  \label{eqn:fc_net}
\end{equation}
where $\mW_1 \in \R^{m \times d}$, $\mW_l \in \R^{m \times m}$, $l = 2,
\dots, L-1$ and $\mW_L \in \R^{1 \times m}$, and $m$ is the width of the
network.

Let $\|\vtheta\|_\uN$ be a general norm of the parameters $\vtheta$, we make the
following assumptions about $\|\vtheta\|_\uN$.

\begin{assumption}\label{assu:norm}
  For any set of parameters $\vtheta$, let $f(\cdot; \vtheta)$ be a neural
  network associated with $\vtheta$. Then, there exists a function $\psi(d, L,
  m)$, such that the Rademacher complexity of the set $\cF^Q_{L,m} = \{f(\cdot;
  \vtheta): \|\vtheta\|_\uN \le Q\}$ can be bounded by
  \begin{equation}
    \hat R(\cF^Q_{L,m}) \le Q \cdot \frac{\psi(d, L, m)}{\sqrt n},
    \label{eq:assu_rad_bound}
  \end{equation}
  where $d$ is the dimension of $\vx$, $L$ and $m$ are the 
   neural network depth and width respectively.
\end{assumption}

The above Rademacher complexity bound implies the following a posteriori
estimate. 

\begin{theorem}[A posteriori estimate]\label{thm:apost_gen}
  Let $n$ be the number of training samples. Consider parameters $\vtheta$ of a
  network with depth $L$ and width $m$. Let $\Ls(\vtheta)$ and
  $\hat\Ls(\vtheta)$ be the truncated population risk and empirical risk defined
  in (\ref{eqn:loss}). Then for any $\delta \in (0, 1)$, with probability at
  least $1 - \delta$ over the random choice of training samples, we have
  \begin{equation}
    \left| \Ls(\vtheta) - \hat\Ls(\vtheta) \right| \le
    2 (\|\vtheta\|_\uN + 1) \frac{2 \psi(d, L, m) + 1}{\sqrt n}
    + 2 \sqrt{\frac{2 \log(7 / \delta)}{n}}.
  \end{equation}
\end{theorem}

The proof of Theorem~\ref{thm:apost_gen} follows the same way as for the proof
of Theorem~\ref{thm:apost}. With the a posteriori estimate, we obtain an a
priori estimate by formulating a regularized problem, and comparing the solution
of the regularized problem to a reference solution with good approximation
property. 

\begin{theorem}[A priori estimate]\label{thm:apriori_gen}
  Under the same conditions as in Theorem~\ref{thm:apost_gen}, for $\lambda \ge
  4 + 2/\psi(d,L,m)$, assume that $\hat\vtheta$ is an minimizer of the
  regularized model
  \begin{equation}
    \min_\vtheta \cJ(\vtheta) := \hat\Ls(\vtheta)
    + \lambda \|\vtheta\|_\uN \cdot \frac{\psi(d, L, m)}{\sqrt n},
  \end{equation}
   Then, for any $\delta \in (0, 1)$, with probability at least $1 -
  \delta$ over the random training samples,
  \begin{equation}
    \Ls(\hat\vtheta) \le \Ls(\tilde\vtheta)
    + \left( \|\tilde\vtheta\|_\uN + 1 \right)
    \frac{(4 + \lambda) \psi(d, L, m) + 2}{\sqrt n}
    + 4 \sqrt{\frac{2 \log(14 / \delta)}{n}}.
  \end{equation}
  where $\tilde\vtheta$ is an arbitrary set of parameters for the same
  hypothesis space.
\end{theorem}

Next, we apply this general framework to the $l_1$ path
norm~\cite{neyshabur2015norm}, spectral complexity
norm~\cite{bartlett2017spectrally} and variational
norm~\cite{barron2018approximation}. The definitions of the norms are given
below.

\begin{description}
  \item[$l_1$ path norm] For a residual network defined by~(\ref{eqn:resnet}),
    the $l_1$ path norm~\cite{neyshabur2015norm} is defined as 
    \begin{equation}
      \|\vtheta\| = \big\| {|\vu|}^\intercal
      (\mI + |\mU_L| |\mW_L|) \cdots (\mI + |\mU_1| |\mW_1|) |\mV| \big\|_1,
      \label{eqn:l1_path_norm}
    \end{equation}
  \item[Spectral complexity norm] For a fully-connected
    network~(\ref{eqn:fc_net}), the spectral complexity norm proposed
    in~\cite{bartlett2017spectrally} is given by
    \begin{equation}
      \|\vtheta\|_\uN = \left[ \prod_{l=1}^L \|\mW_l\|_\sigma \right]
      {\left[ \sum_{l=1}^L
      \frac{\|\mW_l^\intercal\|_{2,1}^{2/3}}{\|\mW_l\|_\sigma^{2/3}}
      \right]}^{3/2},
      \label{eqn:bartlett_norm}
    \end{equation}
    where $\|\cdot\|_\sigma$ denotes the matrix spectral norm and
    $\|\cdot\|_{p,q}$ denotes the $(p, q)$ matrix norm $\|\mW\|_{p,q} =
    \|(\|\mW^{:,1}\|_p, \dots, \|\mW^{:, m}\|_p)\|_q$.
  \item[Variational norm] For a fully-connected network~(\ref{eqn:fc_net}), the
    variational norm proposed in~\cite{barron2018approximation} is
    \begin{equation}
      \|\vtheta\|_\uN = \frac{1}{L} \sqrt{V}
      \sum_{l=1}^L \sum_{j_l} \sqrt{V_{j_l}^\text{in} V_{j_l}^\text{out}},
      \label{eqn:barron_norm}
    \end{equation}
    where
    \begin{align*}
      V & = \big\| |\mW_L| \cdots |\mW_1| \big\|_1, \\
      V_{j_l}^\text{in}
      & = \big\| |\mW_l^{j_l, :}| |\mW_{l-1}| \cdots |\mW_1| \big\|_1, \\
      V_{j_l}^\text{out}
      & = \big\| |\mW_L| \cdots |\mW_{l+1}| |\mW_l^{:, j_l}| \big\|_1.
    \end{align*}
\end{description}

\begin{table}[t]
  \centering
  \caption{%
    Comparison of the a posteriori and a priori estimates for different norms
  }\label{tab:compare}
  \small
  \begin{tabular}{c c c c c}
    \toprule
    Norm & Weighted path norm & $l_1$ path norm & Spectral norm
    & Variational norm \\
    \midrule
    A posteriori
    & $\cO \left( \frac{1}{\sqrt n} \right)$
    & $\cO \left( \frac{2^L}{\sqrt n} \right)$
    & $\cO \left( \frac{1}{\sqrt n} \right)$
    & $\cO \left( \frac{L^{3/2}}{\sqrt n} \right)$ \\
    A priori
    & $\cO \left( \frac{1}{Lm} + \frac{1}{\sqrt n} \right)$
    & $\cO \left( \frac{1}{Lm} + \frac{2^L}{\sqrt n} \right)$
    & $\cO \left( \frac{1}{Lm} + \frac{{(L m)}^{3/2}}{\sqrt n} \right)$
    & $\cO \left( \frac{1}{Lm} + \frac{L^{3/2} \sqrt m}{\sqrt n} \right)$
    \\ \bottomrule
  \end{tabular}
\end{table}

When applying Theorem~\ref{thm:apriori_gen}, for residual networks, we choose
$\tilde{\vtheta}$ to be the solution given by Theorem~\ref{thm:approx}, which is
the same solution used in our main theorem in Section~\ref{sec:sketch}. For
fully-connected networks, we slightly modify the construction of $\tilde\vtheta$
(see the appendix for details), such that the a priori estimates we obtain for
different norms all have the same approximation error. But as
$\|\tilde{\vtheta}\|_\uN$ and $\psi$ vary for different norms, the estimation
error comes out differently. To this end, let us recall the expressions of
$\psi$ for the norms mentioned above 

\begin{align*}
  \text{$l_1$ path norm}:
  & \qquad \psi(d, L, m) = 2^L \sqrt{2 \log 2m}, \\
  \text{Spectral norm}:
  & \qquad \psi(d, L, m) = 12 \log n \sqrt{2 \log 2m}, \\
  \text{Variational norm}:
  & \qquad \psi(d, L, m) = L \log n \sqrt{(L-2) \log m + \log(8ed)}.
\end{align*}
On the other hand, one can derive following bounds for
$\|\tilde{\vtheta}\|_\uN$ (see the appendix for details):
\begin{align*}
  \text{$l_1$ path norm}:
  & \qquad \|\tilde\vtheta\|_\uN \le 4\|f^*\|_\cB, \\
  \text{Spectral norm}:
  & \qquad \|\tilde\vtheta\|_\uN \le 16 {(L m)}^{3/2} \|f^*\|_\cB, \\
  \text{Variational norm}:
  & \qquad \|\tilde\vtheta\|_\uN \le 4 \sqrt m \|f^*\|_\cB.
\end{align*}
Plugging the results above into Theorem~\ref{thm:apriori_gen}, we get a priori
estimates of the regularized model using different norms. The results are
summarized in Table~\ref{tab:compare}. They are shown in the order of $L$, $m$
and $n$, the logarithmic terms are ignored. The notation $\cO(\cdot)$ hides
constants that depend only on the target function. We see that the weighted
path norm is the only one in which the second term in the a priori error bound
scales cleanly as $\cO(1 / \sqrt n)$, i.e., it is independent of the depth $L$.

Note that in Table~\ref{tab:compare} the standard $l_1$ path norm gives an a
priori estimate with an exponential dependence on $L$, different from the case
for the weighted path norm. To see why, consider a network $f(\cdot; \vtheta)$
with $\vtheta = \{\mV, \mW_l, \mU_l, \vu\}$. By the Rademacher complexity bound
associated with the weighted path norm (\ref{eqn:rad_resnet}), this function is
contained in a set with Rademacher complexity smaller than
\begin{equation}
  \frac{C_1}{\sqrt n}
  \big\| {|\vu|}^\intercal
  (\mI + 3 |\mU_L| |\mW_L|) \cdots (\mI + 3 |\mU_1| |\mW_1|) |\mV| \big\|_1.
  \label{eqn:rad_path1}
\end{equation}
On the other hand, if we use the $l_1$ path norm, this function is contained in
a set with Rademacher complexity smaller than
\begin{equation}
  \frac{C_2}{\sqrt n}
  \big\| {|\vu|}^\intercal
  (2 \mI + 2 |\mU_L| |\mW_L|) \cdots (2 \mI + 2 |\mU_1| |\mW_1|) |\mV| \big\|_1,
  \label{eqn:rad_path2}
\end{equation}
where $C_1$ and $C_2$ are constants. This gives rise to the exponential
dependence. This is not the case in (\ref{eqn:rad_path1}) as long as the
weighted path norm is controlled.

The use of the variational norm eliminates the exponential dependence for the
complexity bound, but still retains an algebraic dependence. 

The story for the spectral norm is different. It was shown
in~\cite{bartlett2017spectrally} that the Rademacher complexity of the
hypothesis space with bounded spectral norm has an optimal scaling
($1/\sqrt{n}$). However, as the depth of the network goes to infinity, this
hypothesis space shrinks to $0$ if the bound on the spectral norm is fixed.
Therefore, in order to get the desired bound on the approximation error, one has
to increase the bound on the spectral norm (the value of $Q$). This again
results in the $L$ dependence in the estimation error.

When deriving the results in Table~\ref{tab:compare}, we used a specific
construction $\tilde{\vtheta}$ to control the approximation error. Other
constructions may exist. However, they will not change the qualitative
dependence of the estimation error, specifically the dependence (or the lack
thereof) on $L, m$ in the second term of these bounds, the term that controls
the estimation error.

\section{Conclusion}%
\label{sec:conclusion}

We have shown that by designing proper regularized model, one can guarantee
optimal rate of the population risk for deep residual networks. This result
generalizes the result in~\cite{e2018priori} for shallow neural networks.
However, for deep residual networks, the norm used in the regularized model is
much less obvious.

From a practical viewpoint, it was demonstrated numerically
in~\cite{e2018priori} that regularization improves the robustness of the
performance of the model. Specifically, the numerical results
in~\cite{e2018priori} suggest that the performance of the regularized model is
much less sensitive to the details of the optimization algorithm, such as the
choice of the hyper parameters for the algorithm, the initialization, etc. We
expect the same to be true in the present case for regularized deep residual
networks. In this sense, the regularized models behave much more nicely than
un-regularized ones. One should also note that the additional computational cost
for the regularized model is really negligible.

The present work still does not explain why vanilla deep residual networks,
without regularization, can still perform quite well. This issue of ``implicit
regularization'' still remains quite mysterious, though there has been some
recent progress for understanding this issue for shallow
networks~\cite{brutzkus2017sgd, li2018learning, allen2018learning}. Regarding
whether one should add regularization or not, we might be able to learn
something from the example of linear regression. There it is a standard practice
to add regularization in the over-parametrized regime, the issue is what kind of
regularized terms one should add. It has been proven, both in theory and in
practice, that proper regularization does help to extract the appropriate
solutions that are of particular interest, such as the ones that are sparse. For
neural networks, even though regularization techniques such as dropout have been
used sometimes in practice, finding the appropriate regularized models and
understanding their effects has not been the most popular research theme until
now. We hope that the current paper will serve to stimulate much more work in
this direction.

\bibliographystyle{plainnat}
\bibliography{resnet_apriori}

\appendix

\section{The missing details in Section~\ref{sec:compare}}

\subsection{Approximation properties of deep fully-connected networks}

Consider a deep fully-connected network with depth $L$ and width $m$
(\ref{eqn:fc_net})
in the form:
\[
  f(\vx; \vtheta) = \mW_L \sigma(\mW_{L-1} \sigma(\cdots \sigma(\mW_1 \vx)))
\]
where $\mW_1 \in \R^{m \times d}$, $\mW_l \in \R^{m \times m}$, $l = 2, \dots,
L-1$ and $\mW_L \in \R^{1 \times m}$. 
Taking the same approach as in Theorem~\ref{thm:approx}
and~\cite{e2018exponential}, we construct the deep fully-connected network from
a two-layer network.
From Theorem~\ref{thm:approx_shallow}, there exists a two-layer network with
width $M$, such that
\[
  \left\| \sum_{j=1}^M a_j \sigma(\vb_j^\intercal \rvx) - f^*(\rvx) \right\|^2
  \le \frac{16 \|f^*\|_\cB^2}{M}
\]
and
\[
  \sum_{j=1}^M |a_j| \|\vb_j\|_1 \le 4 \|f^*\|_\cB.
\]
Since the ReLU activation $\sigma(\cdot)$ is positively homogeneous, we can
assume without loss of generality that $a_1 = a_2 = \cdots = a_M = a \le 4
\|f^*\|_\cB$ and $\|\vb_1\|_1 + \|\vb_2\|_1 + \cdots + \|\vb_M\|_1 = 1$. Now let
$M = (m - d) (L - 1)$, and rewrite the subscripts as $\vb_{l,j} =
\vb_{(m-d)(l-1)+j}$, $l = 1, \dots, L-1$, $j = 1, \dots, m-d$. 
Define a fully-connected network $f(\cdot; \tilde\vtheta)$ by
\begin{gather*}
  \mW_1 = \begin{bmatrix}
    \mI_d \\ \vb_{1,1}^\intercal \\ \vdots \\ \vb_{1,m-d}^\intercal
  \end{bmatrix}, \quad
  \mW_l = \begin{bmatrix}
    \mI_d & 0 \\
    \vb_{l,1}^\intercal & \\
    \vdots & \mI_{m-d} \\
    \vb_{l,m-d}^\intercal &
  \end{bmatrix},\ l = 2, \dots, L-1, \\
  \mW_L = \begin{bmatrix}
    0 & 0 & \cdots & 0 & a & a & \cdots & a
  \end{bmatrix},
\end{gather*}
then it is easy to verify that $f(\vx; \tilde\vtheta) = a \sum_{j=1}^M
\sigma(\vb_j^\intercal \vx)$. This ensures that the approximation property of
fully-connected multi-layer neural network is at least as good as the two-layer
network.

\subsection{Calculation of the spectral complexity norm}

Recall the spectral complexity norm (\ref{eqn:bartlett_norm}) proposed
in~\cite{bartlett2017spectrally}
\[
  \|\vtheta\|_\uN = \left[ \prod_{l=1}^L \|\mW_l\|_\sigma \right]
  {\left[ \sum_{l=1}^L
  \frac{\|\mW_l^\intercal\|_{2,1}^{2/3}}{\|\mW_l\|_\sigma^{2/3}} \right]}^{3/2}.
\]
For $l = 1, \dots, L-1$, the matrix spectral norm satisfies $\|\mW_l\|_\sigma
\ge 1$, and
\[
  \|\mW_l\|_\sigma - 1 \le \|\mW_l - \mI\|_\sigma \le \|\mW_l - \mI\|_F
  = {\left[ \sum_{j=1}^{m-d} \|\vb_{l,j}\|_2^2 \right]}^{1/2}
  \le \sum_{j=1}^{m-d} \|\vb_{l,j}\|_1,
\]
thus
\[
  \prod_{l=1}^{L-1} \|\mW_l\|_\sigma
  \le \prod_{l=1}^{L-1} \left[ 1 + \sum_{j=1}^{m-d} \|\vb_{l,j}\|_1 \right]
  < e
\]
since $\sum_{l=1}^{L-1} \sum_{j=1}^{m-d} \|\vb_{l,j}\|_1 = 1$. 
The $(p, q) = (2, 1)$ matrix norm satisfies
\[
  \|\mW_l^\intercal\|_{2,1} =
  \big\| (\|\mW_l^{1,:}\|_2, \dots, \|\mW_l^{:,m}\|_2) \big\|_1
  = d + \sum_{j=1}^{m-d} \sqrt{1 + \|\vb_{l,j}\|_2^2}
  < \sqrt2 m.
\]
In addition,
\[
  \|\mW_L\|_\sigma = \|\mW_L\|_{2,1} = \|\mW_L\|_2 = a \sqrt{m - d}
  \le 4 \|f^*\|_\cB \sqrt{m - d}.
\]
Therefore, the spectral complexity norm
satifies
\[
  \|\tilde\vtheta\|_\uN \le
  e \cdot 4 \|f^*\|_\cB \sqrt{m - d} \cdot L^{3/2} \cdot \sqrt2 m
  \le 16 {(L m)}^{3/2} \|f^*\|_\cB.
\]

\subsection{Calculation of the the variational norm}

Recall the variational norm (\ref{eqn:barron_norm}) proposed
in~\cite{barron2018approximation}
\[
  \|\vtheta\|_\uN = \frac{1}{L} \sqrt{V}
  \sum_{l=1}^L \sum_{j_l} \sqrt{V_{j_l}^\text{in} V_{j_l}^\text{out}},
\]
where
\begin{align*}
  V & = \big\| |\mW_L| \cdots |\mW_1| \big\|_1, \\
  V_{j_l}^\text{in}
  & = \big\| |\mW_l^{j_l, :}| |\mW_{l-1}| \cdots |\mW_1| \big\|_1, \\
  V_{j_l}^\text{out}
  & = \big\| |\mW_L| \cdots |\mW_{l+1}| |\mW_l^{:, j_l}| \big\|_1.
\end{align*}
Notice that for any $l$,
\[
  \sum_{j_l=1}^m V_{j_l}^\text{in} V_{j_l}^\text{out} = V.
\]
Therefore
\[
  \|\vtheta\|_\uN \le \frac{1}{L} \sqrt V \cdot L \cdot \sqrt{m V} = \sqrt m V.
\]
Now it is easy to verify that
\[
  V = a \sum_{l=1}^{L-1} \sum_{j=1}^{m-d} \|b_{l,j}\|_1
  = a \le 4 \|f^*\|_\cB.
\]

\end{document}